\documentclass[letterpaper, 10 pt, conference]{ieeeconf}  

\IEEEoverridecommandlockouts                              

\usepackage{graphics} 
\usepackage{epsfig} 
\usepackage{amsmath} 
\usepackage{amssymb}  

\usepackage[ruled, linesnumbered]{algorithm2e}

\usepackage{stfloats}
\usepackage{adjustbox}
\newtheorem{lemma}{Lemma} 
 
\usepackage{cite}

\usepackage{hyperref}
\hypersetup{
    colorlinks=true,
    linkcolor=blue,      
    urlcolor=blue
}

\newtheorem{thm}{Theorem}

\newtheorem{prop}[thm]{Proposition}

\usepackage{subfigure}
\usepackage{graphicx}
\usepackage{xcolor}
\usepackage{soul}

\title{\LARGE \bf
Safety-Critical Control under Uncertainty via Adaptive Conformal Quantile Prediction Intervals
}

\author{Hao Zhou$^{1}$, Yanze Zhang$^{1}$, Yiwei Lyu$^{2}$ and Wenhao Luo$^{2}$
\thanks{$^*$This work was supported in part by the U.S. National Science Foundation under Grant 2530297.}
\thanks{$^{1}$The authors are with the Department of Computer Science, University of Illinois Chicago, Chicago, IL, 60607, USA.
Email: {\tt\small \{hzhou134, yzhan361\}@uic.edu}%
}
\thanks{$^{2}$The authors are with the Department of Computer Science and Engineering, Texas A\&M University, College Station, TX, 77843, USA.
Email: {\tt\small \{yiweilyu, wenhaol\}@tamu.edu}%
}
}

\begin{document}

\maketitle
\thispagestyle{empty}
\pagestyle{empty}

\begin{abstract}
Safety-critical control under uncertainty requires uncertainty representations that are both statistically valid (for certifiable performance)  
and compatible with enforceable safety constraints. However, existing methods often assume particular distributions of uncertainty for provable safety guarantees or establish symmetric and input-agnostic prediction intervals for robust safety, which can lead to misaligned or overly conservative safety constraints in control synthesis. In this paper, we introduce a novel safe control framework with adaptive uncertainty quantification that constructs calibrated and state-dependent prediction intervals to enable high-probability safety guarantees, while improving constrained control performance.
The framework leverages adaptive conformal prediction (ACP) and extends it with conformal quantile regression (CQR) to capture distribution-free, asymmetric uncertainty intervals with certifiable probabilistic coverage, and integrates the resulting uncertainty sets into a probabilistic control barrier function formulation to enforce robust safety with reduced conservativeness. This yields uncertainty-aware safe control constraints that can be incorporated within a model predictive control (MPC) framework to provide provably safe behaviors with high probability. Simulation and theoretical results are provided to demonstrate the effectiveness of our approach. 
\end{abstract}

\section{Introduction}\label{sec: introduction}
The robotic systems are increasingly being deployed in the real world, such as autonomous driving \cite{he2022autonomous, lyu2021probabilistic} and quadrotor control~\cite{wang2017safe, khan2020barrier, zhou2026certifiable}. In those scenarios, the robot needs to interact with the complex environment while avoiding unsafe behaviors. However, ubiquitous uncertainties in real world conditions, such as friction variations, actuator errors, and state-estimation noise, introduce stochastic disturbances that perturb the system dynamics and are difficult to capture accurately using first-principles models due to their complexity and time-varying nature. As a result, guaranteeing safety under such uncertainty remains a central challenge for controller synthesis in safety-critical robotic control.

To handle safety constraints in the presence of uncertainty,
work in \cite{kothare1996robust} assumes bounded uncertainty, where the bound is heuristically chosen and sufficiently large to ensure safety. However, such a design could lead to overly conservative control actions and therefore reduced task efficiency. Moreover, predefined bounded uncertainty may not match up with realistic conditions, such as unbounded uncertainty (e.g., Gaussian distributions) or time-varying uncertainty. To account for more realistic uncertainty models, 
safe control under uncertainty is typically formulated as a Model Predictive Control (MPC) with chance constraints given the distribution of the uncertainty as a prior~\cite{blackmore2010probabilistic, blackmore2011chance, ono2008iterative, zhu2019chance}. For example, work in~\cite{blackmore2011chance, ono2008iterative} assume linear dynamics and linear chance constraints, which enable efficient uncertainty propagation and analytical quantification under Gaussian noises. However, the assumption of Gaussian noises could fail to capture uncertainties in the real world (e.g., time-varying uncertainty).  Sampling-based methods, such as the probabilistic particle\cite{blackmore2010probabilistic}, were proposed to quantify the uncertainty of the chance constraint without depending on a particular distribution and the robot system. However, these methods usually require extensive computation that is intractable on resource-limited hardware to obtain a precise uncertainty quantification. 

Recently, Conformal Prediction (CP)~\cite{lei2015distribution, papadopoulos2002inductive}, as a statistical method, provides an efficient distribution-free approach for uncertainty quantification and has gained increasing attention in robotics for quantifying the uncertainty as chance constraints~\cite{cleaveland2024conformal, lindemann2023safe, yang2023safe, zhou2024safety, zhou2025computationally}. However, the exchangeability assumption limits its applicability in online control settings, where the data distribution can evolve over time, especially when the robot is interacting with a changing environment. Adaptive Conformal Prediction (ACP)~\cite{gibbs2021adaptive, gibbs2024conformal} addresses this by adaptively adjusting the uncertainty quantification online, thereby improving robustness when the distribution shifts and providing asymptotic coverage guarantees. Building on that, work in \cite{dixit2023adaptive, zhou2024safety} combined the ACP with the MPC to achieve safe control without relying on assumptions about uncertainty distributions. However, both \cite{zhou2024safety} and \cite{dixit2023adaptive} adopt the symmetric residual score function to quantify the uncertainty, which uses a single global error scale for all inputs and potentially renders a wider confidence interval, often causing overly conservative control behaviors.

To address these challenges, this paper proposes a novel uncertainty-aware safe control framework that adaptively quantifies uncertainty through calibrated asymmetric confidence intervals in an online, distribution-free manner. 
These intervals establish unequal deviations around the prediction values that preserve high-probability safety guarantees when integrated into control constraints, while enjoying reduced conservativeness.

The \textbf{contributions} are threefold: 
\begin{itemize}
    \item We couple the Conformalized Quantile Regression (CQR) with Adaptive Conformal Prediction (ACP) to quantify the uncertainty of safe constraints under unknown motion noise, which yields adaptive and certifiable prediction intervals for control synthesis.    
    
    \item We integrate the uncertainty-aware safety constraint into a probabilistic extension of discrete-time control barrier functions (DT-CBF) and present a model predictive control (MPC) that leverages adaptively quantified uncertainty to generate control actions with high-probability safety guarantees and improved control efficiency.
        
    \item Theoretical analysis and simulation results are provided to justify the enforced high-probability safety and improved task efficiency with our proposed method. 
\end{itemize}

\section{Preliminaries}\label{Sec: prelim and prob}

\subsection{Discrete-Time Control Barrier Functions}  \label{sec: MPC-CBF}
Consider the nonlinear discrete-time control system,
\begin{equation}\label{eq:dt dynamics}
    x_{k+1}=f(x_k, u_k)
\end{equation}
where $x_k \in \mathbb{R}^n$ and $u_k \in \mathbb{R}^m$ represent the system state and control input at the time step $k$, respectively. $f: \mathbb{R}^n \times \mathbb{R}^m \mapsto \mathbb{R}^n$ is a continuous function.

Given the system dynamics in Eq.~\eqref{eq:dt dynamics} and the safe set $\mathcal{C} \subset \mathbb{R}^n$, $\mathcal{C} \triangleq \{x_k\in \mathbb{R}^n | h(x_k) \geqslant 0\}$ where $h: \mathbb{R}^n \mapsto \mathbb{R}$ is a continuous function, the \textbf{discrete-time control barrier function} (DT-CBF)\cite{agrawal2017discrete} has been used to enforce the system state remains in the safe set $\mathcal{C}$ if its initial state satisfies $x_0 \in \mathcal{C}$. We summarize it as the following lemma. 

\begin{lemma}\label{lem:dt-cbf}[Discrete-time Control Barrier Functions (summarized from~\cite{zeng2021safety})]
Given the control-affine system defined in Eq.~\eqref{eq:dt dynamics} and the safe set $\mathcal{C}$, the admissible control space for any Lipschitz continuous controller
$u_k \in \mathbb{R}^m$ 
at the time step $k$ rendering $\mathcal{C}$ forward invariant is defined as below, 
\begin{align}\label{eq:sampledata_cbc_lemma}
    \mathcal{H}(x_k) = \{u_k \in \mathbb{R}^{m} | h(f(x_k, u_k)) - h(x_k) \geqslant -\gamma h(x_k) \} \notag
\end{align}
where $\gamma \in (0,1]$ is a user-defined parameter.
\end{lemma}

\subsection{Chance-Constrained Safety}
Lemma.~\ref{lem:dt-cbf} establishes the condition for rendering the safe set forward invariant under a deterministic discrete-time system as in Eq.~\eqref{eq:dt dynamics}.
In practice, however, robotic systems are often affected by stochastic disturbances, in which the deterministic DT-CBF condition can no longer be enforced exactly as the next state is no longer deterministic. As a result, guaranteeing strict safety becomes challenging, especially when the disturbance is governed by an unknown distribution. Consider the nonlinear discrete-time control system under the disturbance,
\begin{equation}\label{eq: stochastic dynamics}
    \hat{x}_{k+1}=f(\hat{x}_k, u_k) + \epsilon_k
\end{equation}
where $\hat{x}_{k} \in \mathbb{R}^n$ represents the system state, and $u_k \in \mathbb{R}^m$ is the control input. $\epsilon_k \in \mathbb{R}^n$ is the process noise. We make no parametric assumption on the distribution of $\epsilon_k$ following the more realistic setting.

Under such stochastic disturbances, ensuring the safety condition holds deterministically often leads to overly restrictive behavior or even infeasibility. Instead, we adopt a chance-constrained safety formulation that enforces safety with a prescribed probability level. In this case, given a user-defined failure probability $\alpha \in (0, 1)$ and considering $\hat{x}_k$ as a random variable, we have:
\begin{equation}\label{eq: prob h}
    \mathrm{Pr}(h(\hat{x}_{k+1}) \geqslant 0) \geqslant 1-\alpha.
\end{equation}
where $\mathrm{Pr}(\cdot)$ is the probability under the event $(\cdot)$.

Given the stochastic dynamics in Eq.(\ref{eq: stochastic dynamics}), Lemma.~\ref{lem:dt-cbf} and Lemma 2 in work \cite{zhou2024safety}, the probabilistic safety constraint in Eq.~\eqref{eq: prob h} can be satisfied via the following chance constraints over control $u_k$:
\begin{equation}\label{eq: B2h}
       \mathrm{Pr}(u_k \in {\mathcal{H}}(\hat{x}_{k})) \geqslant 1-\alpha
\end{equation} 

\subsection{Adaptive Conformal Prediction}\label{sec: ACP with CQR}
\textbf{Conformal Prediction} (CP) \cite{papadopoulos2002inductive}, as a statistical method to formulate a certified region $R(X^{*})$ with guaranteed probability, is used to cover the ground-truth model prediction given any predictive model $F: X \mapsto Y$.
$F$ is trained from dataset $\mathcal{D}=\{(X_k, Y_k)\}$ where $X_k \in \mathbb{R}^n$, $Y_k \in \mathbb{R}$, and $k=1,\dots, N$. Given the model input $X^*$, this probability for the true value $Y^*$ regarding the certified region $R(X^{*})$ can be formulated as,
\begin{equation}\label{eq: cp covarage pr}
	\mathrm{Pr}( Y^{*} \in R(X^{*};F, S^{(r)}) \geqslant 1-\alpha
\end{equation}

where $\alpha \in (0,1)$ is the user-defined failure probability and $Y^{*} \in \mathbb{R}$ is the ground-truth label. $R(X^{*};F, S^{(r)})$ is constructed from the prediction model $F$ and the $1-\alpha$ quantile number $S^{(r)}\in\mathbb{R^{+}}$.where $S^{(r)}$ obtained from the calibration dataset  
$\Bar{\mathcal{D}} =\{(\Bar{X}_k, \Bar{Y}_k)\}_{k=1}^{\Bar{N}}$ plays a key role in establishing the coverage guarantee for the ground-truth output $Y^{*}$. Specifically, $S^{(r)}$ is obtained from a conformal score set $\mathcal{S}$, where each conformal score $S_k$ measures the absolute deviation between the ground truth $\Bar{Y}_k$ and the prediction $F(\Bar{X}_k)$
\begin{equation}\label{eq: rs score prelim}
    S_{k} = |\Bar{Y}_k -  F(\Bar{X}_k)|
\end{equation}

The conformal scores in $\mathcal{S}=\{S_k\}_{k=1}^{\Bar{N}}$ are then sorted in non-decreasing order. Let $S^{(r)}$ denote the $r$-th order statistic, where $r:= \lceil (\overline{N}+1)(1-\alpha) \rceil$ and $\lceil \cdot \rceil$ denotes the ceiling function. Given a new input $X^{*}$, the certified region is defined as $R(X^{*})=[F(X^{*})-S^{(r)}, F(X^{*})+S^{(r)}]$. Under the standard exchangeability assumption between the calibration samples and the new sample $(X^*, Y^*)$, the conformal prediction region satisfies the marginal coverage\cite{papadopoulos2002inductive},
\begin{equation}\label{eq: preli-cp-prob}
	\mathrm{Pr}(F(X^{*})-S^{(r)} \leqslant Y^{*} \leqslant F(X^{*})+S^{(r)}) \geqslant 1-\alpha
\end{equation}

$S^{(r)}$ determines the size of the prediction region $R(X^*)$. $R(X^*)$ has subsequently been adopted in \cite{dixit2023adaptive, zhou2024safety, zhou2025computationally} for obtaining the uncertainty-aware safety constraint. However, the prediction region obtained from Eq.~\eqref{eq: rs score prelim} yields a symmetric confidence interval around the nominal prediction $F(X^*)$, which may introduce unnecessary conservatism and consequently compromise control performance. To address this limitation, this paper proposes a novel method in section \ref{sec: method} that constructs an asymmetric confidence interval by separately characterizing the upper and lower uncertainties, thereby reducing conservatism and improving control efficiency.

\textbf{Adaptive Conformal Prediction.}
To guarantee the probability of the prediction value inside a valid region $R(X^{*})$ under the distribution-shift scenarios, the adaptive conformal prediction (ACP) \cite{gibbs2021adaptive} is proposed to re-estimate the quantile number $S^{(r)}$ online. The main idea is to utilize the learning rate $\eta$  to adjust the failure probability $\alpha_k \in (0,1)$ over time, which can be formulated as
	\begin{equation}
		\label{eq: ACP update}
		\alpha_{k+1} = \alpha_{k} + \eta (\alpha-e_k) \; \mathrm{with}\; e_k=
		\begin{aligned}
			\begin{cases}
				0, \mathrm{if}\; Y_k \in R(X_k;F,S^{(r)}_k), \\
				1, \text{otherwise}.
			\end{cases}
		\end{aligned}
	\end{equation}
It is noted that if $Y_k \in R(X_k;F,S^{(r)})$, then $\alpha_{k+1}$ will increase to undercover the region induced by $S_k^{(r)}$. 

\subsection{Problem Statement} 

Consider a robot system governed by stochastic dynamics in Eq.~\eqref{eq: stochastic dynamics}, our goal is to achieve the safety-critical control under the distribution-free assumption about the stochasticity. To this end, we pose the following MPC problem with chance constraints:
\begin{subequations}\label{eq: MPC-CBF-prob}
\begin{align}
    \quad &\min_{{u}_{k:k+H-1 | k}}\,\, \mathbb{E}[J(\hat{x}_{k:k+H|k},  {u}_{k:k+H-1|k})] \label{eq: E_J} \\ 
        s.t.& \quad
            \hat{x}_{k+\tau+1|k} = f(\hat{x}_{k+\tau|k}, {u}_{k+\tau|k})+\epsilon_{k+\tau|k}, \label{eq:dynamics} \\
            & \quad \mathrm{Pr}(u_{k+\tau|k} \in \mathcal{H}) \geqslant 1-\alpha_{k+\tau|k}, \label{eq: chance constraint}\\
            & \quad \hat{x}_{k|k} \sim \mathbb{P}(\bar{x}_{k|k}) \\
            & \quad {u}_{k+\tau|k} \in [u_{\mathrm{min}}, u_{\mathrm{max}}],\quad\forall \tau=0,\ldots,H-1 
\end{align}
\end{subequations}
where the prediction horizon is denoted by $H$. $\hat{x}_{k+\tau|k}$ denotes the predicted state at horizon step $k+\tau$ given time step $k$, obtained by applying the control input sequence ${u}_{k:k+H-1|k}$ to the system dynamics in Eq.~\eqref{eq:dynamics}. The term $\epsilon_{k+\tau|k}$ represents the process noise, and $\mathbb{P}(\bar{x}_{k|k})$ denotes the unknown distribution associated with the nominal state $\bar{x}_{k|k}$. 
The cost function is defined as $J(\cdot):=l_T(\hat{x}_{k+H|k}) +\sum^{H-1}_{\tau=0} l_s(\hat{x}_{k+\tau|k}, {u}_{k+\tau|k})$ where $l_T(\hat{x}_{k+H|k}), l_s(\hat{x}_{k+\tau|k}, {u}_{k+\tau|k})$ are the terminal cost and stage cost, respectively. $\mathbb{E}[\cdot]$ represents the expectation of the cost function $J(\cdot)$. The chance constraint $\mathrm{Pr}(u_{k+\tau|k} \in \mathcal{H}) \geqslant 1-\alpha_{k+\tau|k}$ arises from the stochasticity of $\hat{x}_{k+\tau|k}$ induced by $\epsilon_{k+\tau|k}$. $u_{\mathrm{min}}$ and $u_{\mathrm{max}}$ specify the control bounds.

Addressing the optimization problem in Eq. (\ref{eq: MPC-CBF-prob}) is particularly challenging in the absence of an explicit distributional assumption about $\epsilon$. Previous work, such as \cite{blackmore2011chance, zhu2019chance}, has relied on the Gaussian distribution to characterize uncertainty in chance constraints, which may violate the real uncertainty distribution or lead to conservative control. In this paper, we focus on solving the following questions,
\begin{itemize}
    \item Q1: How to quantify the uncertainty of the chance constraint under the DT-CBF (Eq.\eqref{eq: chance constraint}) without depending on the particular distribution assumption? 
    \item Q2: With uncertainty quantification, how to obtain a tight (effective) confidence interval that benefits the downstream decision-making (Eq.\eqref{eq: MPC-CBF-prob})?
    \item Q3: How to integrate the effective confidence interval into the MPC problem?
\end{itemize}

Q1 and Q2 will be addressed in section \ref{sec: DT-CBF CQR} while Q3 will be tackled in section \ref{sec: algorithms and analysis}.

\section{Method}\label{sec: method}
\subsection{Conformalized Quantile Regression Score for Uncertainty Quantification} \label{sec: DT-CBF CQR}
The CQR score function, along with the ACP, has been shown to adaptively cover the prediction within the certified region $R$ with a probability guarantee \cite{gibbs2021adaptive, angelopoulos2024theoretical}. We summarize it as follows.

\begin{lemma}\label{lem: CQR}[Summarized from~\cite{gibbs2021adaptive, angelopoulos2024theoretical}] 
Given the failure probability $\alpha \in (0,1)$, the lower quantile estimator $F_l:\mathbb{R}^n\mapsto \mathbb{R}$ and upper quantile estimator $F_u:\mathbb{R}^n\mapsto \mathbb{R}$ from the training dataset $\mathcal{D}$ where $l=\alpha/2$ and $u=1-\alpha/2$, define the conformalized quantile regression score function $S_k$ as, 
\begin{equation}\label{eq: CQR score}
    S_k(\Bar{X}_k, \Bar{Y}_k)=\max(F_l(\bar{X}_k)-\bar{Y}_k, \bar{Y}_k-F_u(\bar{X}_k))
\end{equation}
where the paired $(\bar{X}_k, \bar{Y}_k)$ is adopted from the calibration dataset $\bar{\mathcal{D}}$ (validation dataset). Given the input $X^*$, the following probability holds,
\begin{equation}\label{eq: method-cp-prob}
	\mathrm{Pr}(F_l(X^{*})-S^{(r)} \leqslant Y^{*} \leqslant F_u(X^{*})+S^{(r)}) \geqslant 1-\alpha
\end{equation}
where $Y^*$ is the ground-truth of the predictions while $S^{(r)}$ is the quantile number obtained from the conformal score set $\mathcal{S}$ whose element is constructed by the CQR score $S_k$.
\end{lemma}

Different from the residual score (Eq.\eqref{eq: rs score prelim}) defined in section \ref{Sec: prelim and prob}, the confidence interval in Eq.\eqref{eq: method-cp-prob} obtained by the CQR score function (Eq.\eqref{eq: CQR score}) is a nonparametric prediction interval due to the lower quantile estimator $F_l$ and upper quantile estimator $F_u$, which can be adaptively adjusted via the input. While the confidence interval in Eq.\eqref{eq: preli-cp-prob} from the residual score function is globally defined under the calibration dataset $\bar{\mathcal{D}}$, the confidence interval from Eq.\eqref{eq: preli-cp-prob} must be wide enough to almost cover the data distribution everywhere involved with those high variance regions under small $\alpha$, which potentially overestimates the interval under those low-variance regions.

To solve the MPC problem in Eq. (\ref{eq: MPC-CBF-prob}), the challenge is to translate the implicit chance constraint in Eq.~\eqref{eq: chance constraint} under a distribution-free setting to an explicit uncertainty-aware constraint. Common methods assume a particular distribution~\cite{zhu2019chance, blackmore2011chance} for quantifying uncertainty, which can be invalid or can lead to overly conservative control under the distribution-free assumption. In this section, we propose effectively quantifying the uncertainty of the chance constraint derived from the DT-CBF using adaptive conformal prediction with the CQR score function that can self-tighten the confidence interval.

By substituting the stochastic dynamics in Eq.~\eqref{eq: stochastic dynamics} into the chance constraint in Eq.~\eqref{eq: chance constraint}, the DT-CBF-based chance constraint in Eq.~\eqref{eq: B2h} could be written as
\begin{equation}\label{eq: prob DT-CBF}
    \mathrm{Pr}(h(f(\hat{x}_k, u_k) + \epsilon_k) - h(\hat{x}_k) \geqslant -\gamma h(\hat{x}_k)) \geqslant 1-\alpha
\end{equation}
where $\alpha \in (0,1)$ denotes the failure probability specified by the user and $\gamma \in (0, 1]$ is defined in the DT-CBF.

To quantify the uncertainty of Eq.~\eqref{eq: prob DT-CBF}, Eq.~\eqref{eq: prob DT-CBF} will be considered as a confidence interval estimation problem using CQR with ACP. And hence the prediction model $F(\bar{x}_k, u_k):\mathbb{R}^n\times \mathbb{R}^m \mapsto \mathbb{R}$ for the confidence interval is defined as below, 
\begin{equation}
    F(\bar{x}_k, u_k)\triangleq h(f(\bar{x}_k, u_k)) - (1-\gamma)h(\bar{x}_k)
\end{equation}

where $\bar{x}_k$ is the nominal state from the deterministic dynamics with the control input $u_k$ from the MPC. Let 
\begin{equation}\label{eq: cbf truth}
    F^{*}(\hat{x}_k, u_k)\triangleq h(f(\hat{x}_k, u_k)+\epsilon_k) - (1-\gamma)h(\hat{x}_k)
\end{equation}
 represent the ground truth model. Since CQR requires the upper and lower quantiles to construct a confidence interval, using the quantiles of $F$ may lead to bias relative to the actual safety barrier certificates from CBF. Instead of using the quantile of $F$ for constructing the confidence interval (Eq.\eqref{eq: method-cp-prob}), we introduce a new definition of the prediction model $\hat{F}(\bar{x}_k, u_k):\mathbb{R}^n\times \mathbb{R}^m \mapsto \mathbb{R}$.
 \begin{equation}
      \hat{F}(\bar{x}_k, u_k;d)\triangleq F(\bar{x}_k, u_k) + d(\bar{x}_k)
 \end{equation}

 where $d(\bar{x}_k)=|F^{*}(\hat{x}_k, u_k) - F(\bar{x}_k, u_k)|$. Hence, the lower and upper quantile models is parameterized by $d_{\alpha_l}$ and $d_{\alpha_u}$, respectively. Given the definition of $\hat{F}(\bar{x}_k, u_k; d)$, the CQR score function followed by Eq.\eqref{eq: CQR score}  is defined as,
\begin{equation}\label{eq: cbf cqr score}
\begin{aligned}
    S_k = \max(\hat{F}_{\alpha_l}(\bar{x}_k, u_k; d_{\alpha_l})-F^{*}(\hat{x}, u_k), \\
    F^{*}(\hat{x}_, u_k) - \hat{F}_{\alpha_u}(\bar{x}_k, u_k;d_{\alpha_u}))
\end{aligned}
\end{equation}

where $\alpha_l=\alpha/2$ and $\alpha_u=1-\alpha/2$. $d_{\alpha_l}$ and $d_{\alpha_u}$ are the lower and upper quantiles respectively. At time step $k$, $\bar{x}_k$ and $d(\bar{x}_k)$ are aggregated into the training dataset, which is then used to train $d_{\alpha_l}$ and $d_{\alpha_u}$ via conformalized quantile regression. $\hat{F}_{\alpha_l}$ and $\hat{F}_{\alpha_u}$ are defined as
\begin{equation}
    \hat{F}_{\alpha_l}(\bar{x}_k, u_k; d_{\alpha_l})=F(\bar{x}_k, u_k)+d_{\alpha_l}
\end{equation}
\begin{equation}
    \hat{F}_{\alpha_u}(\bar{x}_k, u_k; d_{\alpha_u})=F(\bar{x}_k, u_k)+d_{\alpha_u}
\end{equation}

Given the CQR score function $S_k$, the conformal score set $\mathcal{S}$ and the quantile number $S^{(r)}$ can both be obtained. Hence, the self-tightening confidence interval under the CQR score function is
 \begin{equation}\label{eq: bound cqr}
      R(\bar{x}_k,u_k;S^{(r)})=[-S^{(r)}+\hat{F}_{\alpha_l}, \hat{F}_{\alpha_u}+S^{(r)}]
 \end{equation}

Unlike the residual score function (Eq.\eqref{eq: preli-cp-prob}), which produces fixed-width confidence intervals given the calibration dataset, the CQR score yields adaptive, self-tightening intervals by leveraging a nonparametric construction based on lower and upper quantiles. Benefiting from the nonparametric property, the confidence interval for deriving the uncertainty-aware constraint can be tightened.

Given the confidence interval of the safety barrier certificates(Eq.\eqref{eq: bound cqr}), we let the lower bound of Eq.\eqref{eq: bound cqr} greater than zero, which will lead to the following uncertainty-aware constraint.

Suppose there exists a failure probability $\alpha \in (0, 1]$. Given the system dynamics in Eq.~\eqref{eq: stochastic dynamics}, the quantile number $S^{(r)}$ under the CQR score function $S_k$ where $r:= \lceil (k+1)(1-\alpha_k) \rceil$ with $\alpha_k$ updated via Eq.\eqref{eq: ACP update}, and the prediction model $\hat{F}(\bar{x}_k, u_k; d)$, then quantifying Eq.~\eqref{eq: prob DT-CBF} with $1-\alpha$ probability guarantee leads to the deterministic constraint,
\begin{equation} \label{eq: quantified CBF}
    -S^{(r)} + \hat{F}_{\alpha_l}(\bar{x}_k, u_k; d_{\alpha_l}) \geqslant 0
\end{equation}

Eq.~\eqref{eq: quantified CBF}
translates the chance constraint (Eq.~\eqref{eq: chance constraint}) into a deterministic constraint using the self-tightening confidence interval with the CQR score function. However, this section does not address how to obtain the confidence interval under the MPC framework or how to embed the uncertainty-aware constraint into the MPC, which will be addressed in Section~\ref{sec: algorithms and analysis}, respectively.

\subsection{Safe Control Design and Analysis}\label{sec: algorithms and analysis}
Given uncertainty-aware constraint (Eq.\eqref{eq: quantified CBF}), the MPC optimization problem in Eq.\eqref{eq: MPC-CBF-prob} becomes,
\begin{subequations}\label{eq: MPC-CBF-ACP-OPT}
\begin{align}
    \quad &\min_{{u}_{k:k+H-1 | k}}\,\, J(\bar{x}_{k:k+H|k}, u_{k:k+H-1|k}) \\
        s.t.& \quad
            \bar{x}_{k+\tau+1|k} = f(\bar{x}_{k+\tau|k}, {u}_{k+\tau|k}), \label{eq: dynamics} \\
            & \quad \hat{F}_{\alpha_l}(\bar{x}_{k+\tau|k}, u_{k+\tau|k};d_{\alpha_l}) - S^{(r)}_{k+\tau|k} \geqslant 0 ,  \label{eq: probability-2}\\
            &\quad \bar{x}_{k|k}=\hat{x}_{k|k}  \\
            & \quad {u}_{k+\tau|k} \in [u_{min}, u_{max}],\quad\forall \tau=0,\ldots,H-1 
\end{align}
\end{subequations}
where $\bar{x}_{(\cdot)}$ is the nominal state. Eq.~\eqref{eq: MPC-CBF-ACP-OPT} formulates an online decision-making problem, where the exchangeability assumption can be violated due to the sequential and adaptive nature of the data. To address this issue, ACP is employed to adaptively update the confidence interval. As defined in Eq.~\eqref{eq: ACP update}, updating the failure probability $\alpha_k$ requires access to the ground-truth function $F^{*}$. However, the future ground truth $F^{*}$ within the MPC prediction horizon is unavailable, making it impossible to update $\alpha_k$ in a distribution-free manner based on future observations. Following \cite{dixit2023adaptive}, we will use the time-lagged conformal score computation for CQR score $S_k$, which is presented in Fig.\ref{Fig: cqr score function}, respectively.

\begin{figure}[t]	
	\centering	
	\includegraphics[scale=1, width=0.95\linewidth]{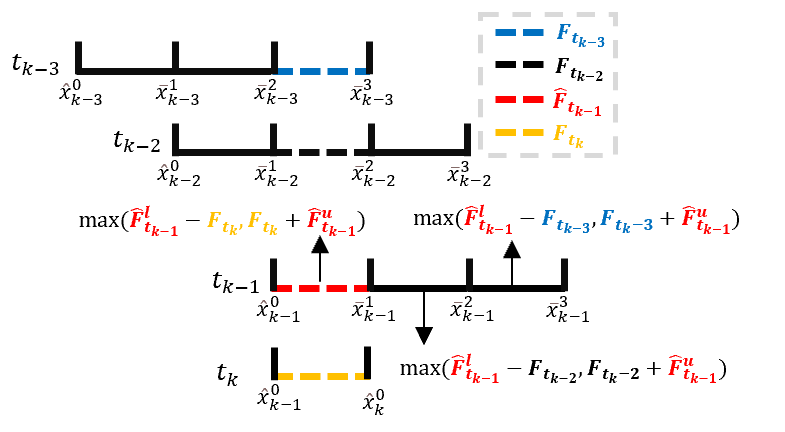}	
	\caption{CQR score computation under MPC. Considering the prediction horizon $H=3$ at time step $t_{k-1}$. Let $F_{t_{k-1-\tau}}=F(\bar{x}_{k-1|k-1-\tau}, u_{k-1|k-1-\tau})$, $\hat{F}^l_{t_{k-1}}=\hat{F}_{\alpha_l}(\hat{x}_{k-1|k-1}, u_{k-1|k-1}; d_{\alpha_l})$, and $\hat{F}^u_{t_{k-1}}=\hat{F}_{\alpha_u}(\hat{x}_{k-1|k-1}, u_{k-1|k-1}; d_{\alpha_u})$. The time-lagged CQR score is defined as $S_{k-1+\tau|k-1} = \mathrm{max}(\hat{F}^l_{t_{k-1}}-F_{t_{k-1-\tau}}, F_{t_{k-1-\tau}}+\hat{F}^u_{t_{k-1}})$ for $\tau = 0, 1, \dots, H-1$. The resulting error terms $S_{k-1+\tau|k-1}$ are then added into the conformal score set $\mathcal{S}$.
 }
	\label{Fig: cqr score function}
\end{figure}

\begin{algorithm}[t]
    \caption{MPC-CBF-ACP-CQR (MCA-CQR)}\label{alg: MCAOPT}
    \KwIn{Parameters: ACP failure probability $\alpha$ and learning rate $\eta$, prediction horizon $H$, $\gamma$ in CBF, total time steps $K$, time step $t_s$, initial safe state $x_0$ and goal state $x_\mathrm{g}$}
    \KwOut{Safe path of the robot }
    \BlankLine
    
    Initialization: $\hat{x}_0=x_0$, \: $S^{(r)}_{k+\tau|k}=0, \tau = 0,\dots, H$\;
    \While{$k \leqslant K$}{
    $\bar{x}_{k+\tau|k},u_k \gets \mathrm{MPC}(\hat{x}_k, S^{(r)}_{k+\tau|k}) \: \tau = 0,\dots, H$ \;
            $\hat{x}_{k+1} \gets f(\hat{x}_k,u_k) + \epsilon_k$ \;
    \ForEach{$\tau < H$}{
            $F(\bar{x}_{k+\tau|k}, u_{k+\tau|k})\gets h(f(\bar{x}_{k+\tau|k}, u_{k+\tau|k})) - (1-\gamma)h(\bar{x}_{k+\tau|k})$\;
     
            $\hat{F}_{\alpha_l} \gets F(\bar{x}_{k+\tau|k}, u_{k+\tau|k})+d_{\alpha_l}$\;
            
            $\hat{F}_{\alpha_u} \gets F(\bar{x}_{k+\tau|k}, u_{k+\tau|k})+d_{\alpha_u}$ \;

            $S_{k+\tau|k} \gets Eq.\eqref{eq: cbf cqr score}$ \;

            $\mathcal{S} \gets S_{k+\tau|k}$ \;
            
            $S^{(r)}_{k+\tau|k} \gets (1-\alpha_{k+\tau|k})$ quantile of $\mathcal{S}$\;

            $\alpha_{k+1+\tau|k+1} \gets Eq.\eqref{eq: ACP update}$ \;
            }
        $d(\bar{x}_k) \gets$ time-lagged residual along the MPC\;
        update $d_{\alpha_l}$ and $d_{\alpha_u}$ using quantile regression \;  
    }
\end{algorithm}

The algorithm for computing the quantile number under ACP to obtain an uncertainty-aware constraint is presented in Algorithm \ref{alg: MCAOPT}. How to integrate the uncertainty-aware constraint into MPC is still presented in Algorithm \ref{alg: MCAOPT}. 

Given the self-tightening confidence interval under the CQR score function, the following proposition states that the proposed algorithm can guarantee the safety constraint within a predefined probability.

\begin{prop}\label{prop: CQR prob guarantee}
     Given the learning rate $\eta$, the initial failure probability $\alpha_0 \in (0,1)$, and the total control time steps $K$, the probability of the onestep-ahead safe constraint will converge to $1-\alpha-p_l$,
    \begin{equation}
        \frac{1}{K}\sum_{k=0}^{K-1} \mathrm{Pr}(F^{*}\in R(\hat{F}, S^{(r)})) \geqslant 1-\alpha-p_l
    \end{equation}
    where $p_l=(\alpha_0 + \eta) / (K\eta)$ is a constant related to the learning rate $\eta$ and the total time step $K$. When $K \rightarrow \infty$, $1-\alpha-p_l$ will converge to $1-\alpha$.
\end{prop}

\begin{proof}
    From proposition 4.1 \cite{gibbs2021adaptive}, we know that
\begin{equation}\label{eq: gibbs prop}
    \frac{\alpha_0 + \eta}{K\eta}\leqslant \frac{1}{K} \sum_{k=0}^{K-1} e_k - \alpha \leqslant \frac{1-\alpha_0 + \eta}{K\eta}
\end{equation}
Given the confidence interval $R(\hat{F}, S^{(r)})$ and the indicator $e_k$, the following probaility equation holds,
\begin{equation}\label{eq: expection of residual}
    \mathrm{Pr}(F^{*}\in R(\hat{F}, S^{(r)})) = \mathbb{E}(1-e_k)=1-\mathbb{E}(e_k)
\end{equation}

Then, taking expectation on both side of Eq.~\eqref{eq: gibbs prop}, we obtain,
\begin{equation}
    \frac{\alpha_0 + \eta}{K\eta}\leqslant \frac{1}{K} \sum_{k=0}^{K-1} \mathbb{E}(e_k) - \alpha \leqslant \frac{1-\alpha_0 + \eta}{K\eta}
\end{equation}
where $K, \eta, \alpha, \alpha_0$ are the constant and hence $\mathbb{E}(\square)=\square$. Substituting $\mathbb{E}(e_k)$ in Eq.\eqref{eq: expection of residual}, we obtain
\begin{equation}
    \frac{\alpha_0 + \eta}{K\eta}\leqslant \frac{1}{K} \sum_{k=0}^{K-1} (1-\mathrm{Pr}(F^{*}\in R(\cdot)) - \alpha \leqslant \frac{1-\alpha_0 + \eta}{K\eta}
\end{equation}
$R(\cdot)$ denotes $R(\hat{F}, S^{(r)})$. Then, we have
\begin{equation}
    1-\alpha-p_l\leqslant \frac{1}{K} \sum_{k=0}^{K-1} \mathrm{Pr}(F^{*}\in R(\hat{F}, S^{(r)})) \leqslant 1-\alpha-p_u
\end{equation}
where $p_l=(\alpha_0+\eta)/(K\eta)$ and $p_u=(1-\alpha_0+\eta)/(K\eta)$. Thus when $K\rightarrow \infty$, $p_l=0$ and $p_u=0$.
\end{proof}

\section{Simulation}\label{Sec: simulation}

\begin{figure}[t]	
	\centering	
	\includegraphics[scale=1, width=\linewidth]{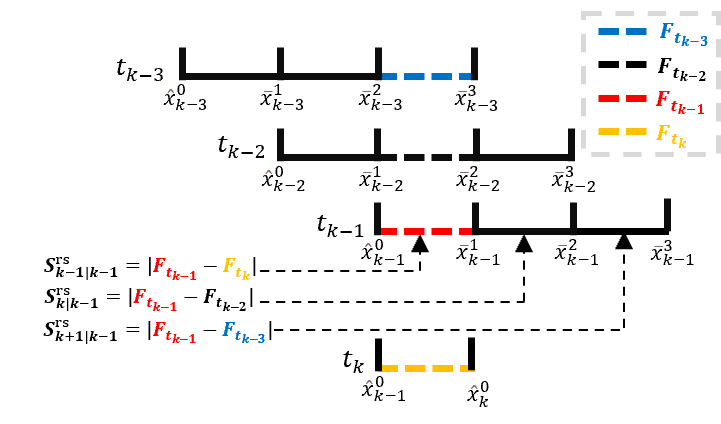}	
	\caption{Residual score computation with DT-CBF under MPC. For the current prediction horizon ($H=3$) starting at time step $t_{k-1}$, the time-lagged CBF residual is defined as $S^{\mathrm{rs}}_{k-1+\tau|k-1} = |F(\hat{x}_{k-1|k-1}, u_{k-1|k-1}) - F(\bar{x}_{k-1|k-1-\tau}, u_{k-1|k-1-\tau}) |$ for $\tau = 1, \dots, H-1$. For $\tau=0$, the residual reduces to $S^{rs}_{k-1|k-1} = |F_{t_{k-1}}-F_{t_k}|$. The resulting error terms ${S^{rs}_{k+\tau|k}}$ are then aggregated into the conformal score set $\mathcal{S}$.
 }
	\label{Fig: residual score function}
    \vspace{-0.5cm}
\end{figure}

This paper proposes an efficient method to quantify the uncertainty of the chance constraint under the DT-CBF within MPC. Using the ACP under the CQR score function, the confidence interval for the uncertainty-aware constraint can be self-tightening. The proposed framework is named as MPC-CBF-ACP-CQR (MCA-CQR).
To validate the effectiveness of MCA-CQR, we compare it with MPC-CBF (MC)~\cite{zeng2021safety}. For the ablation study of the CQR score function, we compare our method (MCA-CQR) with MPC-CBF-ACP using the residual score function, which is named MCA.  Given $F(\bar{x}_k, u_k)$ and $F^{*}(\hat{x}_k, u_k)$ defined in section \ref{sec: DT-CBF CQR}, the residual score function is defined as $S_k^{\mathrm{rs}}=|F(\bar{x}_k, u_k)-F^{*}(\hat{x}_k, u_k)|$ and the time-lagged conformal score computation is presented in Fig.\ref{Fig: residual score function}.
Extensive simulations are conducted to answer the following questions:
\begin{itemize}
    \item \textbf{Q1.} Can the proposed uncertainty-aware framework using the CQR score function guarantee safety within the desired probability?
    \item \textbf{Q2.} What is the advantage of using the CQR score function for downstream decision-making compared to the traditional residual score function?
    \item \textbf{Q3.} Can the proposed method guarantee safety under different forms of distributions, different robot platforms, and different robot control tasks?
\end{itemize}

To answer those questions, we conduct our simulation using a mobile robot and a quadrotor engaged in the navigation and path-tracking task. The distance between the robot and the obstacle is computed to evaluate the safety. To construct the safety constraint, a discrete CBF is applied with the distance between the robot and the obstacle, which is defined as, $h(x_k)=||x_k-x_{\mathrm{obs}}||^2-(r_{\mathrm{rob}}+r_{\mathrm{obs}})^2$ where $x_{\mathrm{obs}}$ is the position of the obstacle, $r_{\mathrm{rob}}$ is the radius of the robot, and $r_{\mathrm{obs}}$ is the radius of the obstacle.

\begin{table*}[h!]
\centering
\caption{Quantitative experiment results under different distributions and different robot platforms}
\label{table: quantitative experiment}
\begin{tabular}{|c|c|c|c|c|c|c|c|}
\hline
 &  & \multicolumn{3}{c|}{Mobile Robot} & \multicolumn{3}{c|}{Quadrotor}  \\ \cline{3-8}
 & Distribution & MC & MCA & MCA-CQR(ours) & MC & MCA & MCA-CQR(ours) \\ \hline
 & Uniform & 0 & 100 & 100 & 45 & 100 & 100 \\ \cline{2-8}
 \%Succ & Gaussian & 0 & 100 & 100 & 41 & 100 & 100 \\ \cline{2-8}
 & Time Varying & 0 & 100 & 100 & 40 & 100 & 100 \\ \hline
 & Uniform & -0.026 & 0.027 & 0.008 & -0.056 & 0.102 & 0.054 \\ \cline{2-8}
$\min h$ & Gaussian & -0.128 & 0.032 & 0.006 & -0.054 & 0.086 & 0.057 \\ \cline{2-8}
 & Time Varying & -0.088 & 0.028 & 0.012 & -0.048 & 0.104 & 0.042 \\ \hline
\end{tabular}

\begin{flushleft}
{\footnotesize 
Quantitative comparison of MC, MCA, and MCA-CQR.
\%Succ. denotes the proportion of experiments in which robots safely finish the corresponding tasks.
Safety is evaluated using the metric $h(x_k)\geqslant 0$; the robot is considered safe if this condition holds, and it is deemed to have collided with an obstacle otherwise. $\min h$ denotes the minimum value of $h(x_k)$ observed across all 100 experiments. Empirical results show that our method not only guarantees safety across different distributions but also approaches the obstacle because of the effective confidence interval provided by the CQR score function.
}
\end{flushleft}
\vspace{-0.5cm}
\end{table*}

\begin{figure*}[!t]	
	\centering
	\subfigure[Uniform distribution: U ]
	{
		\begin{minipage}[t]{0.33\linewidth}
			\centering
			\includegraphics[scale=1, width=\linewidth]{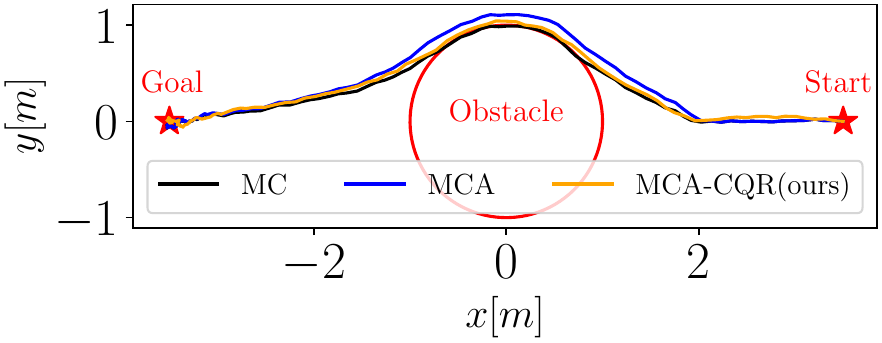}
		\end{minipage}	
	}
	\subfigure[Gaussian distribution: $\mathcal{N}$]
	{
		\begin{minipage}[t]{0.3\linewidth}
			\centering
			\includegraphics[scale=1, width=\linewidth]{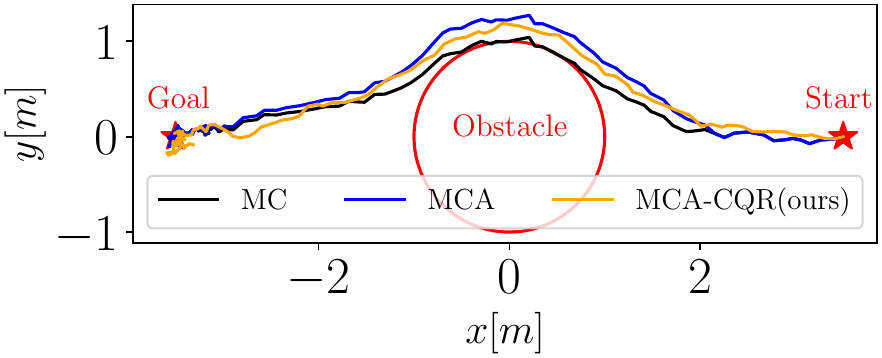}
		\end{minipage}	
	}
    \subfigure[Time varying distribution]
	{
		\begin{minipage}[t]{0.3\linewidth}
			\centering
			\includegraphics[scale=1, width=\linewidth]{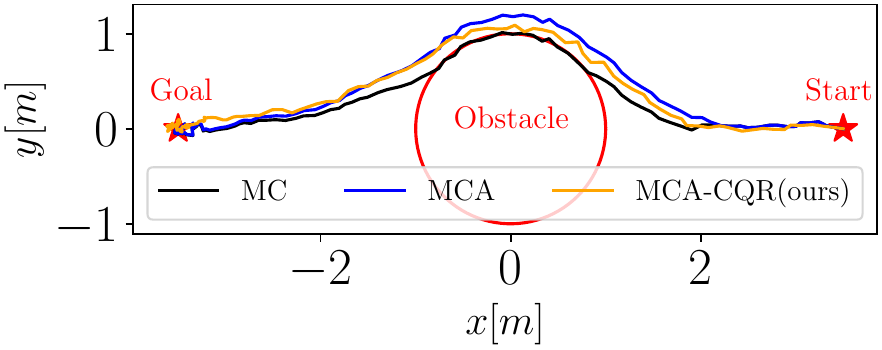}
		\end{minipage}	
	}
	\caption{Simulations for a mobile robot navigating in a single obstacle scenario under different noise distributions with different methods. $\gamma=0.9$ and $\alpha=0.05$ are adopted. Our proposed framework (MCA-CQR) not only achieves safe navigation across different motion noise distributions but also provides a smaller safety margin than the MCA.
    } 
	\label{Fig: mobile robot with distributions}
\end{figure*}

\begin{figure*}[!t]	
	\centering
	\subfigure[Uniform distribution: U ]
	{
		\begin{minipage}[t]{0.3\linewidth}
			\centering
			\includegraphics[scale=1, width=\linewidth]{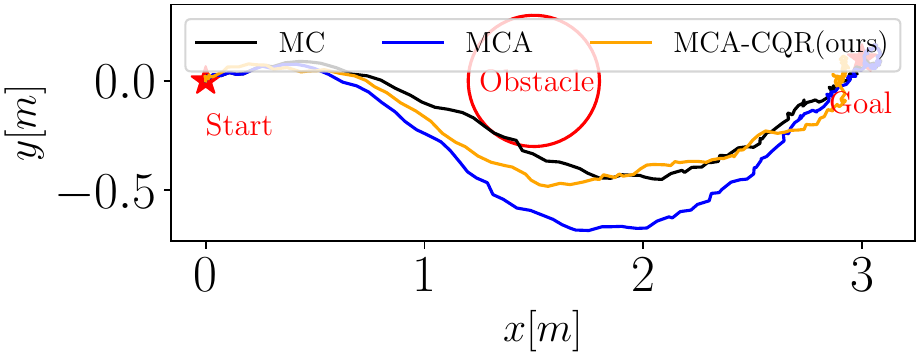}
		\end{minipage}	
	}
	\subfigure[Gaussian distribution: $\mathcal{N}$]
	{
		\begin{minipage}[t]{0.33\linewidth}
			\centering
			\includegraphics[scale=1, width=\linewidth]{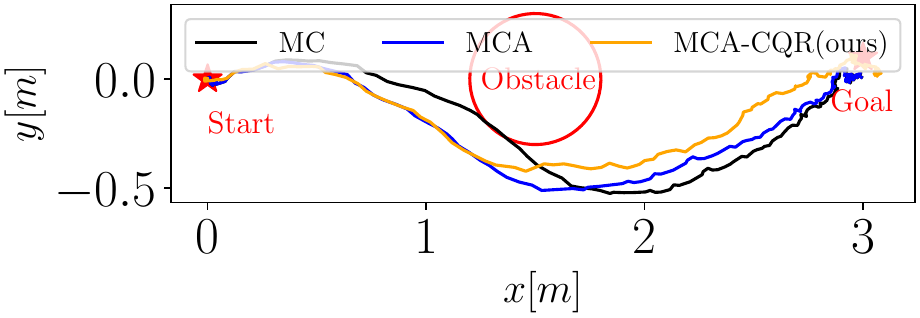}
		\end{minipage}	
	}
    \subfigure[Time varying distribution]
	{
		\begin{minipage}[t]{0.31\linewidth}
			\centering
			\includegraphics[scale=1, width=\linewidth]{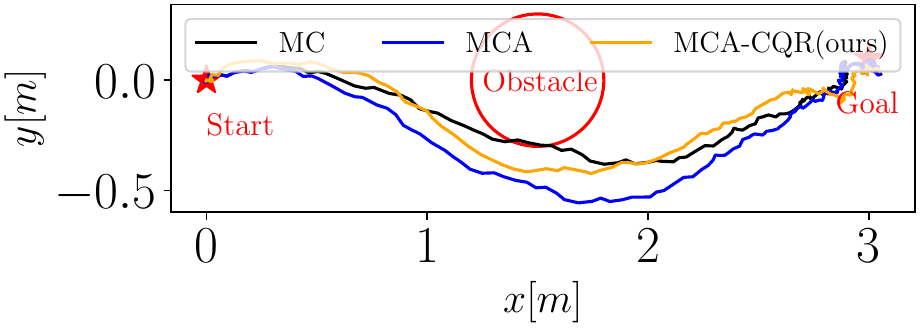}
		\end{minipage}	
	}
	\caption{Simulations for a planar quadrotor navigating in a single obstacle scenario under variant noise distributions with different methods. $\gamma=0.9$ and $\alpha=0.05$ are adopted. Our proposed framework (MCA-CQR) not only achieves safe navigation across different motion noise distributions but also provides a smaller safety margin than the MCA.
    } 
	\label{Fig: quadrotor with different distributions}
    \vspace{-0.5cm}
\end{figure*}

\begin{figure}[t]	
	\centering	
	\includegraphics[scale=1, width=1.0\linewidth]{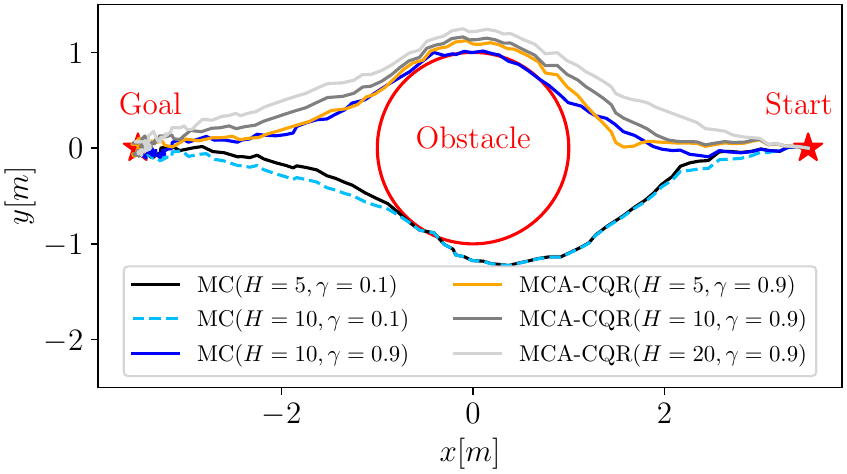}	
	\caption{Parameter analysis with MC\cite{zeng2021safety} and MCA-CQR(ours) under different MPC horizon $H$ and the $\gamma$ in CBF. This is conducted under the mobile robot platform. MC collides with the obstacle across different horizons, even with a small $\gamma$, whereas our method can achieve a collision-free path. In this simulation, we apply the Gaussian distribution or the uniform distribution randomly to $\epsilon_k$ at every step $k$.}
	\label{Fig: mobile parameter}
    \vspace{-0.5cm}
\end{figure}

\begin{figure}[t]	
	\centering	
	\includegraphics[scale=1, width=0.8\linewidth]{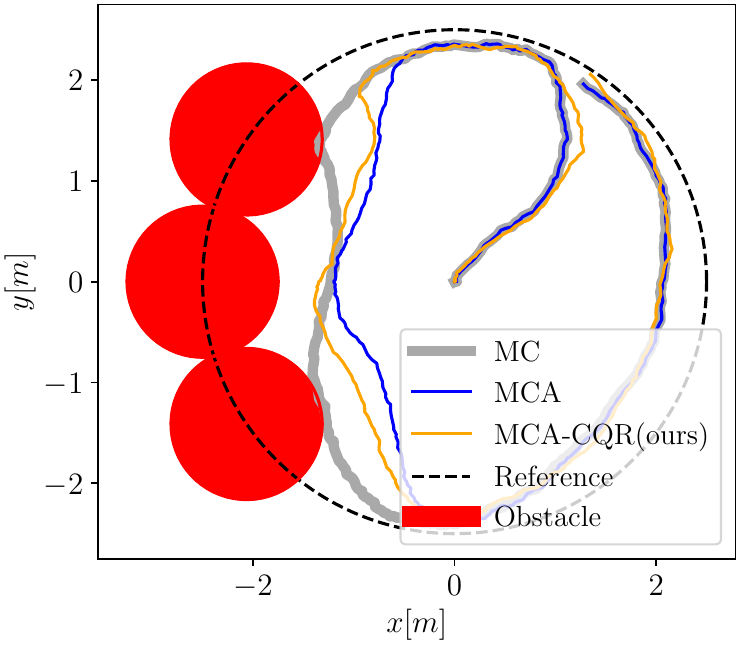}	
	\caption{Simulation for a planar quadrotor to track a circle while avoiding the obstacle. MC, MCA, and MCA-CQR are all applied. Results show that MC collides with the obstacle, whereas the uncertainty-aware method proposed in this paper (MCA-CQR) not only guarantees safety but also achieves a smaller safety margin than MCA under the residual score function.}
	\label{Fig: quadrotor circle tracking}
    \vspace{-0.5cm}
\end{figure}

\subsection{Mobile Robot with Single-integrator Dynamics}
We consider the mobile robot modeled by discrete-time single-integrator dynamics:
\begin{equation} \label{eq: dynamics simulation}
    \hat{x}_{k+1} = \hat{x}_k + u_k\Delta t + \epsilon_k
\end{equation}
where $\epsilon_k \in \mathbb{R}^2$ is the motion uncertainty and $\Delta t=0.02$. Note that we do not impose any parametric distributional assumption on $\epsilon_k$, which means that it might be non-Gaussian or even time-varying.Each component of the control input $u_k$ is bounded within $[-5, 5]$.

\textbf{Simulation parameter.} For the MPC parameter, we consider
the Linear Quadratic Regulator (LQR) cost objective where $Q=10*\mathrm{\textbf{I}}_{2}$ and $R=\mathrm{\textbf{I}}_{2}$. 
Unless otherwise specified, the MPC horizon $H=10$ and $\gamma=0.9$ are adopted for all the mobile robot simulation experiments. The failure probability for ACP is set as $\alpha=0.05$. For visualization, $r_{\mathrm{rob}}=0$ is adopted, while a circular obstacle with radius $r_{\mathrm{obs}}=1$ is located at the origin. The corresponding start and goal points are shown in Fig. \ref{Fig: mobile robot with distributions}.

\textbf{Effect of $H$ and $\gamma$}. 
To analyze the influence of the hyperparameter on the performance of the proposed MCA-CQR, we compare MCA-CQR (our method) with MC baseline under different settings of $H$ and $\gamma$. As shown in Fig.~\ref{Fig: mobile parameter}, even with a smaller $\gamma$ ($\gamma=0.1$), which leads to a more conservative safety constraint and would therefore be expected to reduce the chance of collision, such conservatism alone is insufficient to ensure safety under stochastic dynamics. In contrast, our MCA-CQR can safely navigate to the goal in this stochastic setting regardless of the value of $H$ and $\gamma$.

\textbf{Distribution-free validation.} To answer Question 3 under distribution-free settings, three types of $\epsilon_k$ in Eq.~\eqref{eq: dynamics simulation} are applied: (i) Gaussian noise $\epsilon_k \sim 0.02*\mathcal{N}(\mathbf{0}_2,\mathbf{I}_{2})$, (ii) uniformly distributed noise $\epsilon_k \sim U(-0.02, 0.02)$ with each dimension of the robot state, and (iii) a time-varying distribution where $\epsilon_k$ is randomly sampled from either (i) or (ii) is applied randomly to each time step. This setup is used to test whether the proposed method can maintain the collision avoidance when the noise distribution is unknown and changes over time.

We compare our method (MCA-CQR) against our method with the residual score function (MCA) and  MPC-CBF (MC) \cite{zeng2021safety}. The simulation results in Fig.~\ref{Fig: mobile robot with distributions} show that our proposed method (MCA-CQR) can safely navigate to the goals, whereas MC cannot when facing uncertainty. Furthermore, compared with MCA, the trajectories produced by our MCA-CQR remains closer to the obstacle, suggesting that the confidence interval induced by the CQR score function is less conservative. 

For quantitative evaluation, we run 100 experiments with different random seeds and report the safety rate, and the minimum distance across all experiments in Table \ref{table: quantitative experiment}.

\subsection{Quadrotor}
To further validate the performance of the proposed method on different platforms and control tasks, we additionally consider the quadrotor system for testing both \textbf{navigation} and \textbf{path-tracking} tasks while avoiding collisions with obstacles. 
The planar quadrotor dynamics is defined as,
\begin{equation}
    \begin{aligned}
        m\ddot{x}&=-T\sin\theta \\
        m\ddot{y}&=mT\cos\theta -mg \\
        I\ddot{\theta}&=\tau
    \end{aligned}
\end{equation}
where the control input is thrust $T$ and torque $\tau$. The mass $m=1.0$ and the inertial parameter $I=0.011$ are adopted.

\textbf{Simulation parameter.}
Similar to the mobile robot navigation example, the MPC cost is formulated as the LQR cost objective with $Q=\mathrm{diag}(20,20,1,2,2,0.1)$ and $R=\mathrm{diag}(0.1, 5)$.
Unless otherwise specified, the MPC horizon $H=10$ and $\gamma=0.9$ are adopted for all experiments under the quadrotor simulation. The failure probability for ACP is set to $\alpha=0.05$.

\subsubsection{Navigation} The control task for the quadrotor is to move from the start position to the goal position while avoiding the obstacle. The start point, goal point, and the obstacle location are shown in Fig. \ref{Fig: quadrotor with different distributions}.

\begin{figure}[!t]	
	\centering
	\subfigure[The trajectory of quadrotor under our method ]
	{
		\begin{minipage}[t]{0.8\linewidth}
			\centering
			\includegraphics[scale=1, width=\linewidth]{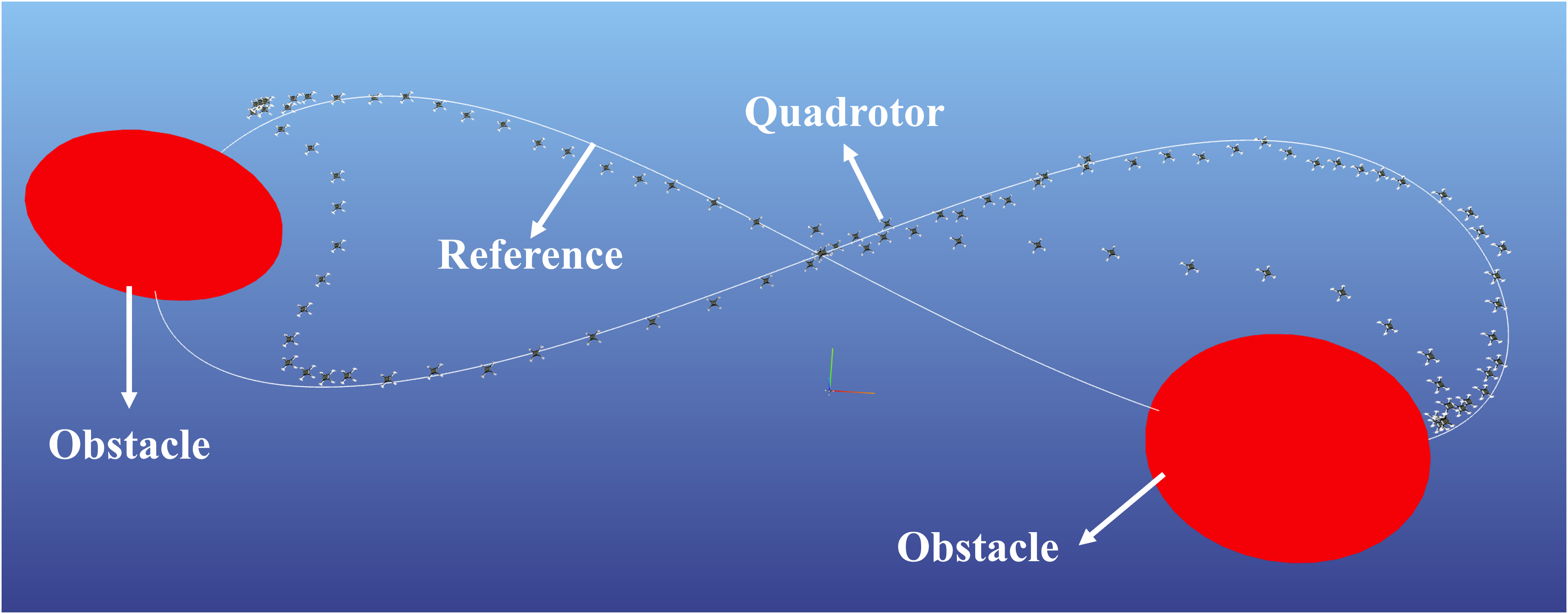}
		\end{minipage}	
	}
	\subfigure[The trajectory under different methods]
	{
		\begin{minipage}[t]{0.9\linewidth}
			\centering
			\includegraphics[scale=1, width=\linewidth]{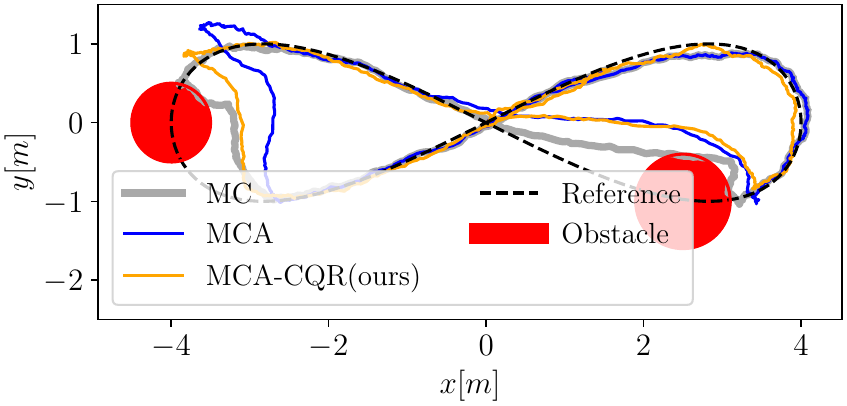}
		\end{minipage}	
	}
    
	\caption{Simulations for a quadrotor to track a figure eight while avoiding the obstacle under the time-varying motion noise. $\gamma=0.9$ and $\alpha=0.05$ are adopted. Our proposed framework (MCA-CQR) not only guarantees safety but also provides a small safety margin with respect to the MCA under the residual score function.} 
	\label{Fig: 3D quadrotor tracking}
\end{figure}

\textbf{Distribution-free validation.}
Similar to the mobile robot navigation task,
three types of $\epsilon_k$ in Eq.~\eqref{eq: dynamics simulation} are applied: (i) Gaussian noise $\epsilon_k \sim \mathcal{N}(\mathbf{0}_6,\mathbf{diag}(0.0005, 0.0005, 0.0001, 0.0005, 0.0005, 0.0001))$, (ii) uniformly distributed noise $\epsilon_k \sim U(-0.02, 0.02)$ with each dimension, and (iii) time varying distribution where the noise disturbance at each step is randomly sampled from either (i) or (ii).

We compare our method (MCA-CQR) against our method with the residual score function (MCA) and MPC-CBF (MC) \cite{zeng2021safety}. Simulation results in Fig. \ref{Fig: quadrotor with different distributions} show that our proposed method can guarantee safety, whereas MC cannot when facing uncertainty. On the other hand, benefiting from the effective confidence interval, the trajectory generated by our method (MCA-CQR) is less conservative than MCA. For quantitative evaluation, we also run 100
experiments with different random seeds and report the safety
rate, and the minimum distance across all these 100 episodes in Table \ref{table: quantitative experiment}.

\subsubsection{Tracking}
We continue to validate the performance of the proposed method in trajectory tracking tasks. Specifically, the planar quadrotor is tasked with tracking a circle while avoiding collision with the obstacles. The MPC horizon is set to $H=25$. As depicted in Fig.~\ref{Fig: quadrotor circle tracking}, the quadrotor does not collide with the obstacle when tracking the circular trajectory.

We then extend it to a 3D quadrotor tracking scenario, where it is required to follow a figure-eight trajectory while avoiding obstacles. Using the same MPC prediction horizon $H=25$, our method remains effective in more complex tracking tasks, achieving safe trajectory tracking in the presence of obstacles. The results are demonstrated in Fig.~\ref{Fig: 3D quadrotor tracking}.

\section{Conclusion}\label{Sec: conclusion}
This paper presented an uncertainty-aware safety-critical control framework that combines adaptive conformal quantile prediction with control barrier function–based safety enforcement. By constructing a calibrated, state-dependent uncertainty interval, the proposed approach enables high-probability safety guarantees while reducing conservatism in control. The integration with model predictive control further supports efficient and reliable decision-making under uncertainty. Theoretical analysis is provided to justify the performance guarantees of the proposed method, and simulation results demonstrate the approach's effectiveness on simulated mobile robot and quadrotor platforms.

\bibliographystyle{IEEEtran}
\bibliography{IEEEabrv,ref}

\end{document}